\documentclass{article}

\usepackage[main, final]{neurips_2026}
\usepackage{microtype}
\usepackage{graphicx}
\usepackage{subcaption}
\usepackage{booktabs} 
\usepackage{pifont}      
\usepackage{xcolor}      
\usepackage{colortbl}    
\usepackage{booktabs}    

\usepackage{hyperref}

\usepackage{amssymb}
\usepackage{mathtools}
\usepackage{amsthm}
\usepackage{booktabs}
\usepackage{multirow}
\usepackage{rotating} 
\usepackage{inconsolata}
\usepackage{amsmath}
\usepackage[utf8]{inputenc}
\usepackage{amsmath,amssymb}
\usepackage{hyperref}

\theoremstyle{plain}
\newtheorem{theorem}{Theorem}[section]
\newtheorem{proposition}[theorem]{Proposition}

\theoremstyle{definition}
\newtheorem{definition}[theorem]{Definition}

\theoremstyle{remark}

\definecolor{MyPink}{HTML}{FADBDF}
\definecolor{MyBlue}{HTML}{DDEDFD}
\definecolor{colorBigBang}{HTML}{FFF9C4} 
\definecolor{colorFriends}{HTML}{E1BEE7} 
\definecolor{colorOffice}{HTML}{E0F7FA}  
\definecolor{colorOurs}{HTML}{FFF3E0}    
\definecolor{codegray}{RGB}{240,240,240}

\title{Proactive Memory for Ad-Hoc Recall over \\Streaming Dialogues}

\author{%
Bingbing Wang \\
Department of Computing \\
The Hong Kong Polytechnic University\\
The School of Computer Science and Technology\\
Harbin Institute of Technology (Shenzhen)\\
\And
Jing Li\\
Department of Computing \\
The Hong Kong Polytechnic University\\
\And
Ruifeng Xu \\
The School of Computer Science and Technology\\
Harbin Institute of Technology (Shenzhen)
}

\begin{document}

\maketitle

\begin{abstract}
Real-world dialogue usually unfolds as an infinite stream. 
It thus requires bounded-state memory mechanisms to operate within an infinite horizon. 
However, existing read-then-think memory is fundamentally misaligned with this setting, as it cannot support ad-hoc memory recall while streams unfold.
To explore this challenge, we introduce \textbf{STEM-Bench}, the first benchmark for \textbf{ST}reaming \textbf{E}valuation of 
\textbf{M}emory.
It comprises over 14K QA pairs in dialogue streams that assess perception fidelity, temporal reasoning, and global awareness under infinite-horizon constraints.
The preliminary analysis on STEM-Bench indicates a critical \textit{fidelity-efficiency dilemma}: retrieval-based methods use fragment context, while full-context models incur unbounded latency. 
To resolve this, we propose \textbf{ProStream}, a proactive memory framework for streaming dialogues built on a hierarchical structure.
It enables ad-hoc memory recall on demand by reasoning over continuous streams with multi-granular distillation. 
Moreover, it employs Adaptive Spatiotemporal Optimization to dynamically optimize retention based on expected utility. 
It enables a bounded knowledge state for lower inference latency without sacrificing reasoning fidelity.
Experiments show ProStream delivers higher reasoning fidelity than prior baselines while maintaining substantially lower latency than full-context alternatives.
\end{abstract}

\section{Introduction}
\label{Introduction}
Large language models (LLMs) have largely advanced dialogue systems in various applications, e.g., customer service \citep{1-li2025performance} and personalized education \citep{2-kasneci2023chatgpt}. 
They usually adopt a long-term memory mechanism that enables systems to maintain continuity and reason across longitudinal sessions \citep{3-zhang2024survey,4-packer2023memgpt}. 
It has drawn significant attention, with many studies on memory retention and contextual reasoning  \citep{maharana2024locomo,wu2024longmemeval,pakhomov2025convomem}.

Despite this progress, most memory mechanisms continue to operate under a \textit{read-then-think} paradigm, which presupposes a static, fully accessible context. 
This paradigm is fundamentally misaligned with the practical system for \textit{streaming} dialogues, where \textit{ad-hoc memory recall} may occur anytime as the dialogue unfolds. 
As shown in Figure \ref{fig:1} (Left), full-context attention incurs prohibitive computational costs and unbounded latency growth; moreover, noise from unorganized contexts can lead to reasoning failures. Consequently, we explore the \textit{streaming memory paradigm} illustrated in Figure \ref{fig:1} (Right). It requires memory to operate as a bounded state within an infinite horizon, reflecting the reality that dialogue unfolds as an infinite stream.

\begin{figure}[!t]
  \centering
  \includegraphics[width=0.9\linewidth]{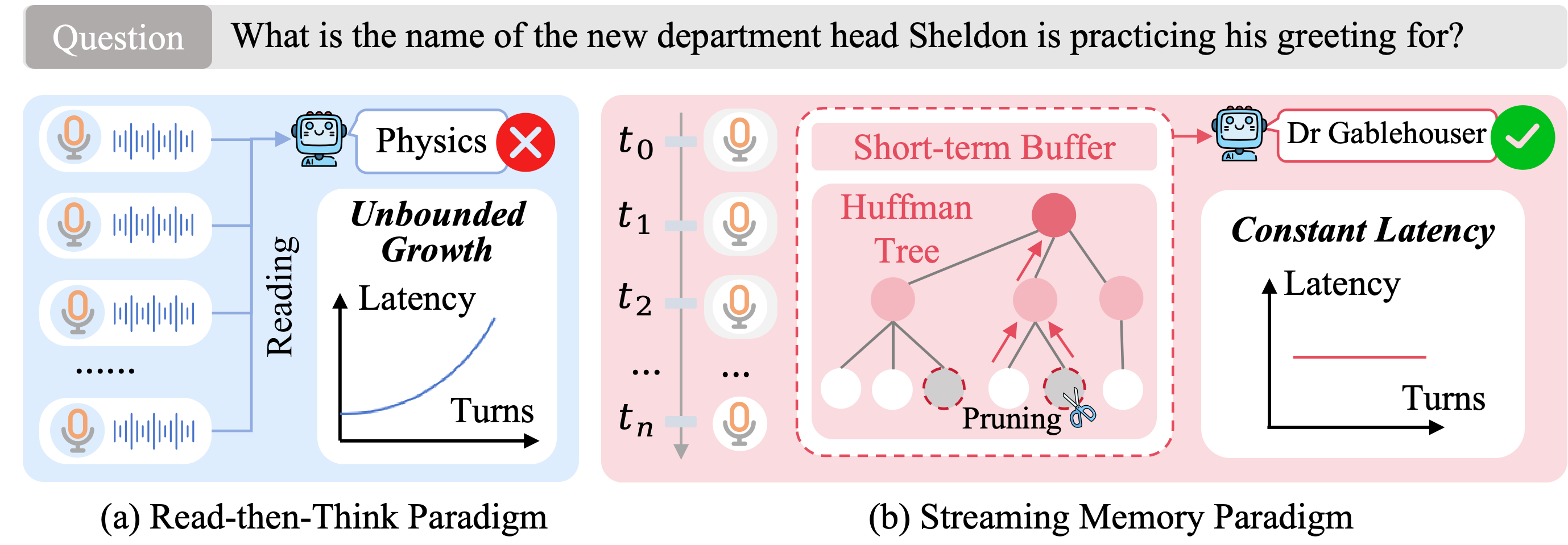}
  \caption{Comparison of the \textit{read-then-think} paradigm (left) and the \textit{streaming memory paradigm} (right) based on The Big Bang Theory (TBBT) dialogues.}
  \vspace{-0.5em}
  \label{fig:1}
\end{figure}

To the best of our knowledge, our study is the first to \emph{explore ad-hoc memory recall in streaming dialogues, which defines the ability to retrieve specific historical context on-demand to satisfy immediate reasoning needs.}
To benchmark this pilot study, we introduce \textbf{STEM-Bench}, the first benchmark for \textbf{ST}reaming \textbf{E}valuation of \textbf{M}emory. Based on LongDialQA \citep{kim2024dialsim}, we convert textual dialogues into synthesized audio to simulate real-world streaming dialogue scenarios. 
STEM-Bench comprises 14k QA pairs to probe three core memory capabilities for ad-hoc recall: high-fidelity perception, structural logical reasoning, and dynamic global awareness. 
Our preliminary analysis of STEM-Bench uncovers a critical \textit{fidelity-efficiency dilemma} challenging current methods. 
Standard memory-based on retrieval approaches fragment context, leading to reasoning degradation. 
Conversely, models attempting to maintain full context suffer from unbounded latency. These findings underscore that high performance in infinite streams cannot be achieved by \textit{reactive} context scanning; it necessitates a paradigm shift toward \textit{proactive} memory organization.

To resolve this, we propose \textbf{ProStream}, a proactive memory framework for streaming dialogues with a hierarchical organization.
As depicted in Figure \ref{fig:1} (Right), ProStream transforms memory into an active, bounded state machine. By synergizing proactive semantic buffering with hierarchical, multi-granular distillation, it facilitates ad-hoc recall whenever demanded by reasoning. 
Moreover, it employs Adaptive Spatiotemporal Optimization to dynamically optimize information retention based on expected utility.
Consequently, it maintains a bounded knowledge state for lower inference latency without sacrificing the reasoning fidelity.
To evaluate ProStream, we first conduct a comprehensive comparison on STEM-Bench against non-trivial baseline memory mechanisms. 
The results demonstrate that ProStream achieves state-of-the-art performance in both accuracy and efficiency, with ablation studies confirming that all modules contribute positively. 
Furthermore, our quantitative analysis indicates that ProStream scales consistently well with LLMs of varying sizes, though it may render higher complexity in certain contexts. 
Finally, we provide a qualitative analysis to offer insights into the model's underlying strengths and remaining limitations.

In summary, our main contributions are as follows:
1) We present the pilot study of ad-hoc memory recall in dialogue streams with \textbf{STEM-Bench}, the first benchmark to evaluate perception, reasoning, and global awareness under strict infinite-horizon constraints.
2) We propose \textbf{ProStream}\footnote{Code is available at https://anonymous.4open.science/r/ProStream-8827}, a novel framework implementing proactive hierarchical memory for streaming dialogues. Leveraging adaptive spatiotemporal optimization to prioritize information with expected utility, it constructs a bounded knowledge topology for ad-hoc, reasoning-demanding recall in continuous streams.
3) Extensive evaluations show ProStream outperforms state-of-the-art baselines in accuracy and efficiency by large margins. Crucially, it replaces the quadratic complexity of existing methods with much lower inference latency, enabling scalable real-time deployment.

\section{Related Work}

\paragraph{Dialogue Benchmarks for Long-Term Memory.}
STEM-Bench falls within this category. This line of research reflects a progressive shift from traditional single-turn modeling \citep{budzianowski2018multiwoz, wei2018airdialogue} toward long-term dialogue understanding.
Early efforts evaluated personalized response generation from extended chat histories \citep{xu2022beyond,xu2022long}. To quantify memory accuracy, subsequent benchmarks adopted the Question Answering (QA) paradigm. Representative studies include {MemoryBank} \citep{zhong2024memorybank} (multi-day histories with 194 probing questions), {LoCoMo} \citep{maharana2024locomo} (single-hop to adversarial reasoning over 50 long-term dialogues), and {PerLTQA} \citep{du2024perltqa} (over 8k questions spanning world knowledge, social relationships, and events). Recently, {LongDialQA} \citep{kim2024dialsim} introduced response latency constraints to penalize slow systems, while {LongMemEval} \citep{wu2024longmemeval} and {ConvoMem} \citep{pakhomov2025convomem} evaluated reasoning, knowledge updates, and abstention capabilities across tens of thousands of QA pairs.
Despite this, existing benchmarks adopt a static, \textit{read-and-think} paradigm diverging from real-world dialogue streams with dynamic, ever-growing contexts. In contrast, STEM-Bench explores \textit{streaming} memory to facilitate ad-hoc recall whenever reasoning demands it within dialogue streams.


\vspace{0.2em}
\noindent\textbf{Methods of Long-Term Memory.}
ProStream aligns with recent efforts on equipping agents with long-term memory. Early approaches expand the context window to process extensive histories directly \citep{fu2024data}, but suffer from prohibitive computation and the ``lost-in-the-middle'' issue \citep{liu2024lost}. Other methods introduce differentiable memory modules \citep{wang2023augmenting} or treat memory as context compression, representing histories as internal states \citep{chevalier2023adapting}, discrete tokens \citep{xu2023recomp}, or retrievable segments via Retrieval-Augmented Generation (RAG) \citep{jimenez2024hipporag}. 
Recent work explores more structured memory mechanisms. {LiCoMemory} \citep{huang2025licomemory} and {SGMem} \citep{wu2025sgmem} model hierarchical or graph-based associations, while {RMM} \citep{tan2025prospect} and {MemGAS} \citep{xu2025towards} employ adaptive retrieval and online reinforcement learning for memory refinement. Event-centric approaches \citep{zhou2025simple} further ground history in structured semantic representations. However, these methods remain largely retrieval-centric and batch-oriented, making them less suitable for streaming scenarios requiring ad-hoc memory recall. In contrast, ProStream formulates memory as online hierarchical state maintenance with proactive spatiotemporal compression, enabling efficient long-term retention under low-latency constraints.

\section{Our STEM-Bench}
\subsection{Problem Formulation}
STEM-Bench formulates dialogue evaluation as an incremental state evolution problem over an infinite input stream. Let $\mathcal{S}=\{I_1,I_2,\dots\}$ denote an unbounded input sequence, where each $I_t$ is either an interaction unit $U_t=\{s_t,l_t\}$ consisting of speaker identity $s_t$ and audio content $l_t$, or a question probe $P_t=\{q_t,y_t\}$ with query $q_t$ issued at time $t$ and the corresponding ground-truth answer $y_t$. 
Under streaming constraints, the model maintains a bounded memory state $\mathcal{M}_t$, updated online from $\mathcal{M}_{t-1}$ and $U_t$. Upon receiving a probe, the model generates an answer $a_t$ based on $\mathcal{M}_t$ and $q_t$, with the goal of maximizing alignment between $a_t$ and $y_t$.


\subsection{STEM-Bench Construction}
Grounded in cognitive and computational theories, STEM-Bench evaluates three core capabilities:
\textbf{(1) High-Fidelity Perception (HFP)} measures atomic detail retention under noise, targeting the \textit{Lost-in-the-Middle} effect \citep{liu-etal-2024-lost} and compression-induced hallucinations through \textbf{Single-hop} and \textbf{Adversarial} tasks.
\textbf{(2) Structural Logical Reasoning (SLR)} assesses the ability to connect fragmented events across timestamps via \textbf{Multi-hop} and \textbf{Comparative} tasks, emphasizing hierarchical evidence composition over shallow retrieval.
\textbf{(3) Dynamic Global Awareness (DGA)} evaluates online maintenance of evolving statistical and temporal states using \textbf{Aggregative} and \textbf{Temporal} tasks, requiring continual state updates without backtracking.

\begin{figure*}[!t]
  \centering
  \includegraphics[width=\linewidth]{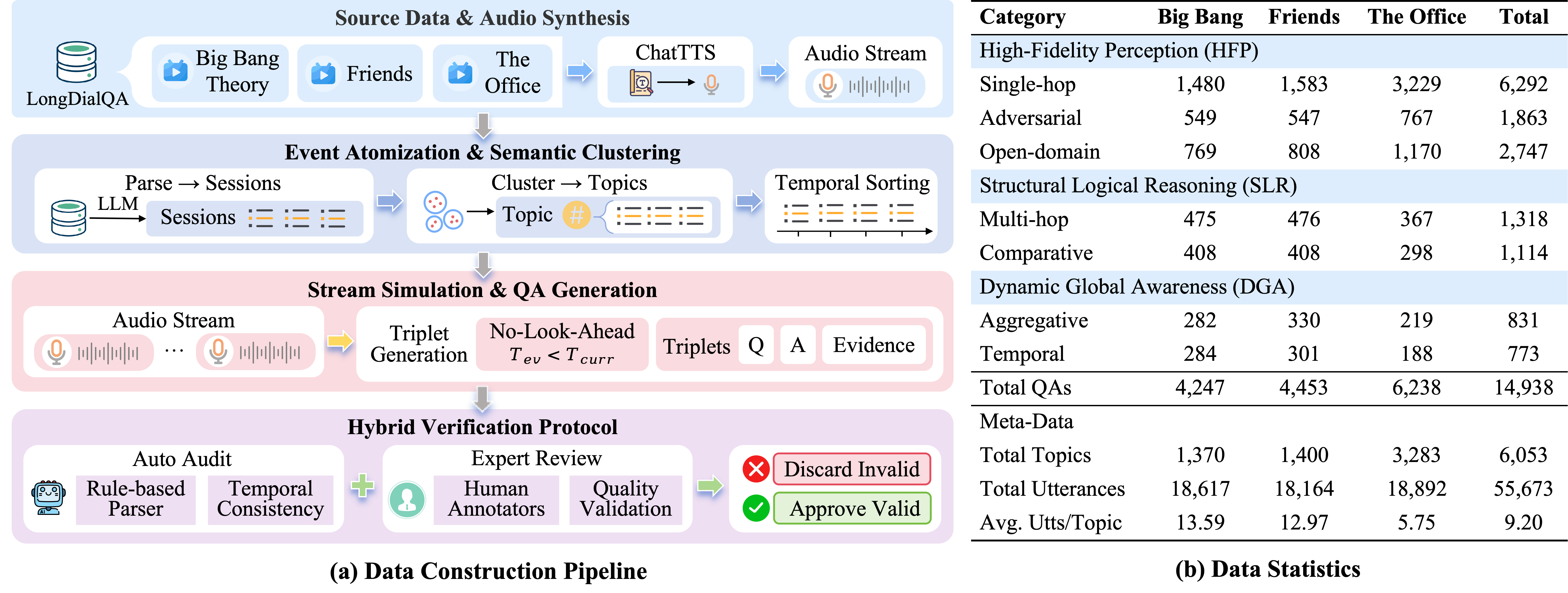}
  \caption{Overview of STEM-Bench Benchmark Curation. (a) Data construction pipeline of the STEM-Bench dataset. (b) Statistics of the STEM-Bench dataset with samples examining HFP, SLR, and DGA. TBBT: The Big Bang Theory.}
  \label{fig:2}
\end{figure*}

\subsection{Data Construction Pipeline}
We implemented a pipeline to transform textual data into a streaming benchmark with audio dialogues, as depicted in Figure \ref{fig:2} (a). The process begins with audio synthesis, converting LongDialQA \citep{kim2024dialsim} into realistic audio via ChatTTS\footnote{https://github.com/2noise/ChatTTS}, followed by semantic clustering to enforce strict causal linearity across sessions.
Subsequently, QA generation phase produces constraint-aware triplets (Question, Answer, Evidence) under a strict \textit{No-Look-Ahead Constraint}, ensuring that for any trigger turn $T_{curr}$, all supporting evidence $T_{ev}$ pertains exclusively to the past $T_{ev}<T_{curr}$.
Finally, data integrity is guaranteed through a hybrid verification protocol.
Following automated rule-based auditing, we recruited three domain experts to validate logical consistency on a random subsample. This review yielded an Inter-Annotator Agreement (Cohen's $\kappa=0.74$) \citep{cohen1960coefficient}, showing good agreements and reliable annotations.


\subsection{Dataset Statistics}
Figure \ref{fig:2} (b) presents statistics for STEM-Bench, comprising 14,938 QA pairs grounded in 55,673 utterances and 6,053 semantic topics across three domains. The distribution is meticulously stratified to probe diverse memory faculties. High-Fidelity Perception forms the foundation with over 6,200 Single-hop and 2,747 Open-domain queries, crucially augmented by 1,863 Adversarial instances designed to expose compression-induced hallucinations. To challenge cognitive depth beyond simple retrieval, Structural Logical Reasoning incorporates 2,432 Multi-hop and Comparative tasks necessitating hierarchical synthesis, while Dynamic Global Awareness forces models to overcome stateless fragility through over 1,600 Aggregative and Temporal instances that demand genuine online state evolution for frequency tracking and chronological sequencing.


\subsection{Evaluation Metric} 
We adopt a holistic evaluation protocol comprising performance and efficiency metrics. 
We utilize BLEU-4 (\textbf{B-4}), ROUGE-L (\textbf{R-L}), and BERTScore (\textbf{B-S}) for lexical and semantic alignment, while Key Entity Matching (\textbf{KEM}) and Evidence Similarity (\textbf{Evid.}) quantify strictly factual precision. Furthermore, we also incorporate Gemini-2.5-Pro (\textbf{Gem}) as an expert judge to evaluate answer correctness and coherence. Efficiency is measured by average inference latency (\textbf{Time} in seconds) to verify real-time feasibility.

\section{Preliminary Analysis}
To identify principled design choices for streaming architectures, we contrast two representative paradigms on our long-context benchmarks: a \textbf{Full-Context Oracle}, which accesses the complete dialogue history, and a \textbf{Standard RAG} baseline utilizing vector-similarity retrieval. 
As shown in Figure\ref{fig:3} (Right), the Full-Context Oracle consistently maintains higher reasoning fidelity, with its average performance metric (dashed purple line) significantly surpassing that of RAG (dashed blue line). The dense distribution of low-accuracy points for RAG, regardless of evidence distance, suggests that retrieval may fragment contexts, making it difficult to capture the global dependencies required for complete reasoning chains.
Conversely, Figure~\ref{fig:3} (Left) exposes the critical \textit{fidelity-efficiency dilemma}. While the Full-Context approach offers superior accuracy, its latency grows unboundedly and exhibits high variance as dialogue length increases, rendering it intractable for real-time interaction. In contrast, RAG maintains low, stable latency but at the cost of substantial context degradation. This tension necessitates a paradigm shift: \textit{an optimal streaming memory must approximate the Oracle's reasoning precision within a strictly bounded computational budget}. Motivated by this objective, we introduce \textbf{ProStream}, a framework specifically engineered to resolve this dilemma through proactive state maintenance, as detailed in the subsequent section.


\begin{figure}[!t]
  \centering
  \includegraphics[width=\linewidth]{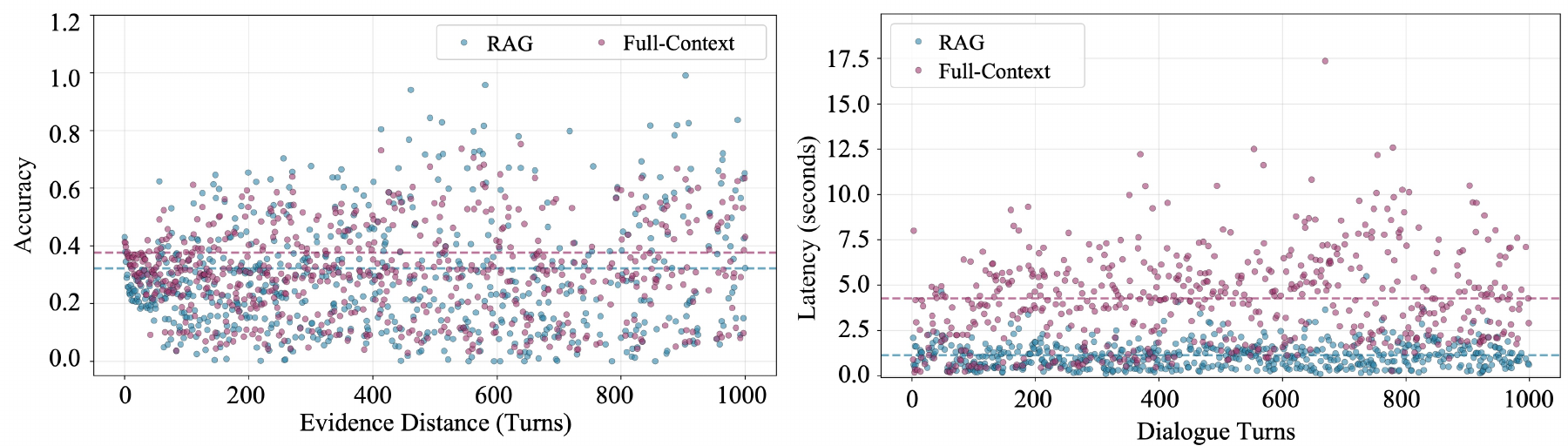}
  \caption{Preliminary analysis of RAG and full-context performance. (Right) The average accuracy of all performance metrics across evidence distances. (Left) Inference latency over dialogue turns. The dashed lines indicate the overall average results.}
  \label{fig:3}
\vspace{-0.6em}
\end{figure}

\begin{figure*}[!t]
  \centering
  \includegraphics[width=0.95\linewidth]{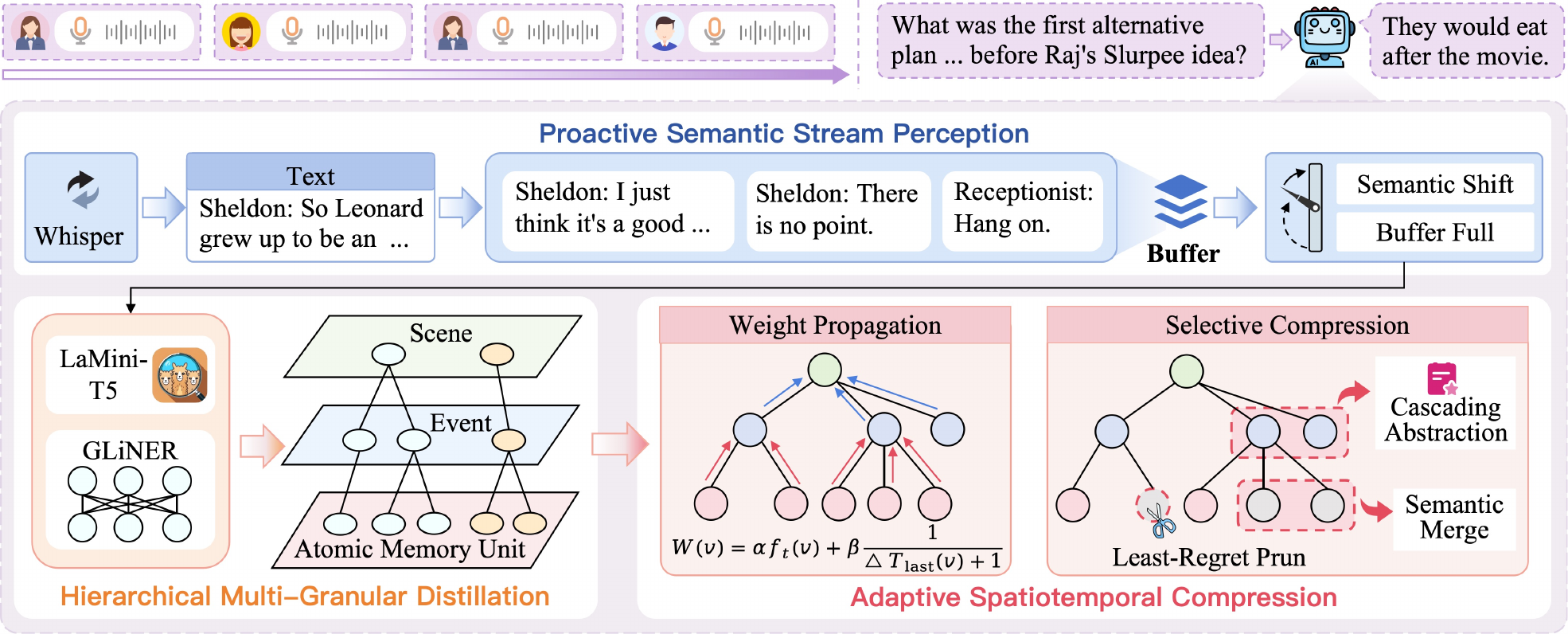}
  \caption{Overview of our ProStream with four components discussed in turn from $\S$\ref{subsec:buffering} to $\S$\ref{subsec:reasoning}.}
  \label{fig:4}
  \vspace{-1em}
\end{figure*}

\section{The ProStream Framework}
As illustrated in Figure~\ref{fig:4}, \textbf{ProStream} reformulates memory maintenance as a bounded state evolution process, transforming infinite input streams into a finite knowledge topology. This pipeline operates in four stages:
\textit{Proactive Semantic Stream Perception} segments continuous raw signals into discrete semantic blocks;
\textit{Hierarchical Multi-Granular Distillation} structures these blocks into a multi-layered tree topology;
\textit{Adaptive Spatiotemporal Optimization} dynamically maximizes information density under strict capacity constraints;
\textit{Probabilistic Evidence-Grounded Reasoning} synthesizes these memory states with immediate contexts for accurate response generation.

\subsection{Proactive Semantic Stream Perception}
\label{subsec:buffering}
Updating global memory $\mathcal{M}$ at every single turn incurs unnecessary computational overhead and risks fragmenting coherent contexts. To address this, we introduce the Short-Term Sensing Buffer (STSB) $\mathcal{B}$. 
Functionally, $\mathcal{B}$ accumulates an incoming stream of interaction units $U_t=\{s_t,l_t\}$. To capture both linguistic content and speaker dynamics, we first transcribe the audio component $l_t$ via Whisper ASR $f_{\text{ASR}}$~\cite{pmlr-v202-radford23a} and concatenate it with the speaker identity $s_t$ to form a composite textual unit $x_t=[s_t;f_{\text{ASR}}(l_t)]$.
We formulate stream segmentation as an online boundary detection problem. Specifically, for each incoming utterance $x_t$, we compute its dense representation $\mathbf{v}_t = f_\text{Enc}(x_t)$ with an encoder $f_\text{Enc}(\cdot)$ and monitor semantic continuity via cosine similarity $\psi_t = \cos(\mathbf{v}_t, \mathbf{v}_{\text{prev}})$ with the preceding unit's embedding.
Once this local coherence drops below a drift threshold $\tau_{\text{drift}}$ or $\mathcal{B}$ reaches capacity, the accumulated sequence is consolidated into a semantic block $\mathcal{T}_{\text{block}}$, ready for the following hierarchical distillation ($\S$\ref{subsec:encoding}).
Here, we further retain a leading context window of size $W_{\text{STSB}}$ from $\mathcal{B}$ for the subsequent cycle, thereby preserving boundary-spanning dependencies during the transition.

\subsection{Hierarchical Multi-Granular Distillation}
\label{subsec:encoding}
Recall the analysis results from Figure \ref{fig:3} implying the necessity for organized memory. 
Inspired by that, the unstructured semantic block $\mathcal{T}_{\text{block}}$ is distilled into a hierarchical tree structure $\mathcal{H}$.
This topology relies on three semantic layers: 1) \textbf{Scene} $c$ for coarse-grained thematic clustering; 2) \textbf{Event} $e$ for temporal context segmentation; and 3) \textbf{Atomic Memory Unit} $o$ for fine-grained factual retention. 

Specifically, the distillation process transforms the input semantic block $\mathcal{T}_{\text{block}}$ into this structure through a dual-path mechanism.
On the generative front, we employ a unified instruction-tuned model $f_{\theta}(\cdot,\cdot)$ \citep{wu2024lamini} to perform recursive summarization. Specifically, conditioned on a summarization instruction $\rho_{\text{evt}}$,  the model first distills the block into a textual event summary $e=f_{\theta}(\mathcal{T}_{\text{block}},\rho_{\text{evt}})$. This summary is subsequently abstracted into a higher-level scene category $c=f_\theta(e,\rho_{\text{scn}})$ using a classification prompt $\rho_{\text{scn}}$.
In parallel, to populate the atomic layer, we utilize GLiNER \citep{zaratiana2024gliner} to parse the block into relational triplets $\epsilon \in \{\mathcal{E}_{\text{sub}},\mathcal{R}, \mathcal{E}_{\text{obj}}\}$ for a subject and an object. To further map these triplets into the tree structure, we treat the discrete subject and object entities as candidate Atomic Memory Unit (AMU) nodes $o$. 
Each candidate $o$ is encoded via $f_{\text{Enc}}(\cdot)$ into $\mathbf{v}_o$. 
$o$ is attached as a child node under the corresponding Event $e_j$ only if it satisfies a strict novelty check against the existing atomic node set $\mathcal{O}_{\text{amu}}$: $\max_{o'\in\mathcal{H}} \cos(\mathbf{v}_o,\mathbf{v}_o')<\theta_{\text{concept}}$ for lightweight structure. Finally, the relation $\mathcal{R}$ is stored as a contextual attribute of these nodes, preserving the semantic dependencies of the original triplet to reconstruct evidence.

\subsection{Adaptive Spatiotemporal Optimization}
\label{subsec:optimization}
Streaming dialogues present a challenge of Online Budgeted Learning. As utterances arrive, the memory topology $\mathcal{H}$ must self-evolve to optimize efficiency for expected utility. Here, we maximize semantic retention under a strict capacity constraint $\mathcal{T}_{\max}$. We formalize this process as a \textit{Recall Probability Maximization} problem, aiming to minimize the expected regret of discarding retrieval-critical information.

\textbf{Probabilistic Utility Formulation.} Drawing on the Rational Analysis of Memory \citep{anderson2013adaptive}, we posit that the optimal memory $\mathcal{H}^*$ at time $t$ contains nodes with the highest posterior of future need. We model the ``need probability" $p(v\,|\,h_{<t})$ of a node $v$ on the memory tree ($\S$\ref{subsec:encoding}), spanning Scenes $c$, Events $e$, and AMUs $o$, given history $h_{<t}$ as a time-decaying Poisson process. We define the utility scalar $u_{v,t}$ as a surrogate for the log-odds of this probability:
\begin{equation}
    u_{v,t} = \underbrace{\alpha \cdot \log(f_{v,t} + 1)}_{\text{Frequency Prior}} + \underbrace{\beta \cdot \exp\left(-\frac{\Delta t_{v}}{\tau}\right)}_{\text{Temporal Drift}},
    \label{eq:utility}
\end{equation}
where $f_{v,t}$ denotes the cumulative access frequency, $\Delta t_v$ is the recency, $\tau$ is the characteristic decay scale, and $\alpha, \beta$ are hyperparameters balancing the two terms.

\textbf{Optimization Objective}.
Let $s_t \in\{0,1\}^{|\mathcal{H}_t|}$ be the state vector indicating retention at time $t$. The classic formulation assumes static values, but in our stream, the value of information decays. We define our objective as maximizing the cumulative discounted information gain over an infinite horizon subject to a dynamic capacity constraint:
\begin{equation}
    \begin{aligned}
    & \max_{\{\mathbf{s}_t\}_{t=1}^\infty} \sum_{t=1}^\infty \gamma^t \sum_{v \in \mathcal{H}_t} s_{v,t} \cdot u_{v,t}
    & \text{s.t.} \quad \sum_{v \in \mathcal{H}_t} s_{v,t} \cdot c_v \leq \mathcal{T}_{\text{max}}, \quad \forall t,
    \end{aligned}
    \label{eq:objective}
\end{equation}
where $\gamma\in (0,1)$ is a discount factor and $c_v$ is the token cost. This is an instance of the Online Knapsack Problem with Decaying Value, known to be NP-hard \citep{hartmanis1982computers}.

\paragraph{The ProStream Policy.}
We approximate the global optimum via a \textit{Greedy Marginal-Utility Policy} $\pi$. When the budget constraint is violated $\sum c_v > \mathcal{T}_{\text{max}}$, $\pi$ restores validity through a synergistic three-stage process. First, it performs \textit{Least-Regret Pruning} by iteratively discarding the node $v^*= \operatorname*{arg\,min}_{v \in \mathcal{H}_t} (u_{v,t}/c_v)$ with the minimal marginal utility density. This greedy heuristic theoretically offers a $(1 - 1/e)$-approximation guarantee assuming monotone submodularity of the information gain. Second, to compress the feature space without semantic loss, the policy applies \textit{Semantic Merging}, collapsing node pairs $\{v_i, v_j\}$ into a single centroid $\mathbf{v}_{\text{new}}$, whenever their latent distance $d(\mathbf{v}_i, \mathbf{v}_j) < \theta_{\text{sim}}$. Finally, to ensure topological consistency, \textit{Cascading Abstraction} recursively prunes any parent node if it loses all supporting leaf nodes, thereby preventing dangling abstractions and maintaining a coherent hierarchy.

\subsection{Probabilistic Evidence-Grounded Generation}
\label{subsec:reasoning}

To generate accurate responses, ProStream synthesizes a unified context
$
\mathcal{K} = \{ \mathcal{B}, \mathcal{M}_{\text{pend}}, \mathcal{M}_{\text{tree}} \},
$
where $\mathcal{B}$ is the STSB, $\mathcal{M}_{\text{pend}}$ acts as an intermediate buffer for distilled yet unformalized semantic representations, and $\mathcal{M}_{\text{tree}}$ comprises the top-$k$ semantic paths retrieved from the hierarchical $\mathcal{H}$. 

\textbf{Hierarchical Retrieval}. To construct $\mathcal{M}_{\text{tree}}$, we employ a top-down traversal strategy. 
We first encode the ad-hoc query q into a dense vector $\mathbf{v}_q=f_{\text{Enc}}(q)$.
Using this representation, the model recursively prunes irrelevant branches starting from the \textit{Scene} level, drilling down to \textit{Event} summaries and finally to \textit{AMUs}.
To ensure that retrieved evidence is both semantically relevant and structurally significant, we compute a composite relevance score for each candidate AMU $o$: $S(o) = \cos(\mathbf{v}_q, \mathbf{v}_o) \cdot u_{o,t}$, where $u_{o,t}$ is the time-variant utility weight defined in Eq.~\ref{eq:utility}. This mechanism effectively filters out low-utility noise to yield high-confidence evidence paths for memory recall.

\textbf{Generation}. 
Given the functional context $\mathcal{K}$, the model generates the final answer sequence $a$ by maximizing
$
P(a \mid q, \mathcal{K})
$,
optimized via the standard autoregressive objective:
\begin{equation}
    \mathcal{J}(\theta) = \sum_{j=1}^{|a|} \log P_{\theta}(y_j \mid y_{<j}, q, \mathcal{K}),
\end{equation}
where $y_j$ denotes the $j$-th output token of answer $a$.
This formulation enforces strict evidence grounding across the full memory spectrum, balancing fidelity and coherence.

\begin{table*}[t]
\centering
\caption{Main comparison results on our STEM-Bench. Best results are \textbf{bolded}, second best are \underline{underlined}. Metrics marked with $\uparrow$ indicate higher is better, while $\downarrow$ denotes lower is better. For ablations (bottom), w/o indicates without. }
\label{tab:main_results}
\tiny
\setlength{\tabcolsep}{0.8pt}
\renewcommand{\arraystretch}{1.0}
\resizebox{\textwidth}{!}{
\begin{tabular}{l  ccccccc  ccccccc  ccccccc}
\toprule
\multirow{2}{*}{\textbf{Method}} & 
\multicolumn{7}{c}{\cellcolor{colorBigBang}\textbf{The Big Bang Theory}} & 
\multicolumn{7}{c}{\cellcolor{colorFriends}\textbf{Friends}} & 
\multicolumn{7}{c}{\cellcolor{colorOffice}\textbf{The Office}} \\
\cmidrule(lr){2-8} \cmidrule(lr){9-15} \cmidrule(lr){16-22}
& Time$\downarrow$ & B-4$\uparrow$ & R-L$\uparrow$ & B-S$\uparrow$ & Evid.$\uparrow$ & KEM$\uparrow$ & Gem$\uparrow$ 
& Time$\downarrow$ & B-4$\uparrow$ & R-L$\uparrow$ & B-S$\uparrow$ & Evid.$\uparrow$ & KEM$\uparrow$ & Gem$\uparrow$ 
& Time$\downarrow$ & B-4$\uparrow$ & R-L$\uparrow$ & B-S$\uparrow$ & Evid.$\uparrow$ & KEM$\uparrow$ & Gem$\uparrow$ \\
\midrule
RAG   & 0.957	& 0.068 & 0.247 & 0.781	&0.355	&0.508	&0.331 
               & \underline{}{0.923} & \underline{0.061} & 0.231 & 0.768	&0.345	&0.502&	0.309
               & 0.918 & \underline{0.074} & \underline{0.271} & 0.778	&0.322	&0.524	&0.385 \\

RQ-RAG         & 5.149&0.051&0.201&0.770&0.275&0.460&0.329
               & 5.098&0.043&0.184&0.758&0.372&0.439&0.313
               & 1.058&0.072&0.270&\underline{0.780}&0.466&	0.535&0.376\\

RAPTOR         & \underline{0.519}	&0.041&	0.165	&0.761	&0.302	&0.375&	0.214
               & \underline{0.513}	&0.035	&0.157	&0.746	&0.396	&0.367	&0.214
               & \underline{0.443}	&0.032 & 0.142	&0.745	&0.324	&0.334	&0.201 \\

Full Context   &7.333	&0.065&	\underline{0.249}&	0.737	&0.250& \underline{0.710}	&\textbf{0.521}
               & 4.870	&0.057	&\underline{0.241}	&0.732&	0.360&	0.706	&\textbf{0.505}	
               &4.776	&0.066&	0.266	&0.742	&0.292	&\textbf{0.731}	&\textbf{0.574} \\

Summary   & 2.001&	0.049	&0.200&	0.783	&0.296&	0.407	&0.278
               &  2.352&	0.042	&0.180	&0.762	&0.334	&0.385	&0.246
               &1.641	&0.048	&0.201	&0.770	&0.365	&0.397&	0.272 \\

GraphRAG       & 0.997	&\underline{0.069}	&0.246&	0.788	&0.417	&0.479	&0.341 
               & 0.950	&0.055	&0.227	&\underline{0.775}	&0.460	&0.465	&0.309 
               & 1.231	&0.055	&0.223	&0.779	&0.420	&0.433	&0.311\\

HippoRAG2     & 4.231&	0.046&	0.176	&0.742&	0.481&	0.480	&0.263
               &6.539	&0.041	&0.165	&0.742	&\underline{0.529}	&0.479	&0.249
               &4.404&	0.040	&0.169	&0.740	&0.499	&0.458&	0.270\\

MemoRAG        & 7.498	&0.049&	0.193	&\underline{0.791}	&\textbf{0.551}&	0.366&	0.265 
               & 4.146	&0.043&	0.178	&0.771	&\textbf{0.575}	&0.355&	0.242
               & 5.922	&0.036	&0.171	&0.773	&\textbf{0.568}	&0.345	&0.225\\
A-Mem          & 1.301	& 0.057& 	0.204& 	0.755	& 0.370&	0.371	& 0.272 
               & 1.290	& 0.053	& 0.200	& 0.742	& 0.386	& 0.376	& 0.266 
               & 1.272 & 	0.059	& 0.220& 	0.753& 	0.353	& 0.377& 	0.295 \\
MemGAS         & 5.711	&0.013	&0.070	&0.528	&0.402&	\textbf{0.768}&	0.508 
               & 7.982	&0.012	&0.072	&0.530	&0.433	&\textbf{0.749}&	0.466 
               & 7.694	&0.006	&0.039&	0.592	&\underline{0.547}	&0.417&	0.285 \\

\midrule
\textbf{ProStream} & \textbf{0.408}	&\textbf{0.118}&	\textbf{0.390}&	\textbf{0.817}&	\underline{0.485}&	{0.612}&	\underline{0.509}
                         &\textbf{0.373}	&\textbf{0.099}	&\textbf{0.360}&	\textbf{0.801}	&0.491&	\underline{0.588}	&\underline{0.476}
                         & \textbf{0.316}&	\textbf{0.114}	&\textbf{0.408}	&\textbf{0.819}&	0.491	&\underline{0.626}	&\underline{0.534}\\

-w/o $\mathcal{B}$    & 0.409&	0.049&	0.182&	0.783	&0.071	&0.369  & 0.255
            & 0.469&	0.034&	0.149	&0.750	&0.276	&0.364  & 0.202 
            &  0.452&	0.030&	0.138	&0.758	&0.191	&0.336 & 0.193 \\
-w/o $\mathcal{M}_\text{tree}$    &0.400&	0.116&	0.382	&0.810	&0.463	&0.607& 0.500
            &0.318&	0.097&	0.304	&0.703	&0.480	&0.581 & 0.466
            & 0.319&	0.110&	0.402	&0.817	&0.484	&0.620 &0.526\\

-w/o $\mathcal{M}_\text{pend}$    & 0.401	&0.115	&0.380	&0.810	&0.475	&0.602 & 0.489
                & 0.320	&0.090	&0.356	&0.704	&0.482	&0.580 & 0.473
                & 0.304	&0.111	&0.401	&0.817	&0.487	&0.617 & 0.528\\

\bottomrule
\end{tabular}
}
\end{table*}

\section{Experiment}
\subsection{Experimental Setting}
\label{Experimental Setting}
\textbf{Implementation Details}. 
For \textit{Proactive Semantic Stream Perception}, we set drift threshold to 0.7, buffer capacity to 5, and $W_{\text{STSB}}$ to 10. In \textit{Hierarchical Multi-Granular Distillation}, node instantiation enforces maximal granularity, with similarity thresholds of 0.85. For \textit{Adaptive Spatiotemporal Optimization}, utility weights $\alpha=0.6$ for frequency, $\beta=0.4$ for recency. For \textit{Probabilistic Evidence-Grounded Generation}, retrieval selects top-5 scenes and top-10 events per scene. The final output includes top-3 AMUs based on a composite relevance score, filtered by a minimum similarity of 0.5.
Detailed configurations are deferred to Appendix~\ref{app:implementation}.

\textbf{Baseline Methods}.
To ensure a comprehensive assessment, we benchmark ProStream against a diverse suite of methods spanning two distinct paradigms:
(1) \textbf{Traditional Retrieval and Summarization Baselines}: Standard RAG, Full-Context, Rolling Summary and advanced structure-enhanced variants, including RQ-RAG \citep{chan2024rq}, RAPTOR \citep{sarthi2024raptor}, GraphRAG \citep{edge2024local}, and HippoRAG2 \citep{gutierrez2025rag}.
(2) \textbf{Agentic Memory Systems}: MemoRAG \citep{qian2025memorag}, A-Mem \citep{xu2025mem}, and MemGAS \citep{xu2025towards}.



\subsection{Experimental Results}\label{ssec:results}
\textbf{Main Results}. We first examine the main comparisons in terms of accuracy and efficiency. 
Table~\ref{tab:main_results} demonstrates that ProStream attains a Pareto-optimal balance of reasoning fidelity and latency. Crucially, it outperforms the Full-Context oracle, challenging the ``more context is better" assumption by demonstrating that proactive distillation effectively mitigates cognitive overload. While ProStream leads in high-order reasoning (Gemini) and generation quality (BLEU-4), it occasionally trails retrieval-centric baselines like MemoRAG in surface-level metrics, such as Evidence Similarity. This validates our design trade-off: prioritizing semantic abstraction over fetching raw, verbatim text. By distilling signals into structured hierarchies, we accept marginal losses in lexical overlap to secure superior reasoning capabilities and scalability essential for real-time streaming. A fine-grained stage-wise cost breakdown of the online ProStream pipeline is provided in Appendix~\ref{app:cost_analysis}. We additionally report transfer evaluations on existing conversation-memory benchmarks in Appendix~\ref{app:benchmarks}.

\textbf{Ablation Study}. Given the overall perforamnce of ProStream, we further analyze the contribution of each component of it, with results in Table \ref{tab:main_results}. Eliminating the Short-Term Sensing Buffer (-w/o $\mathcal{B}$) leads to catastrophic degradation, underscoring its essential role in preventing contextual fragmentation and hallucinations by grounding the model in recent cues. Conversely, removing the Hierarchical Tree (-w/o $\mathcal{M}_\text{tree}$) impairs global consistency and entity tracking, demonstrating that long-term structured retention is indispensable for cross-temporal reasoning. Pending buffer (-w/o $\mathcal{M}_\text{pend}$) proves vital as a synchronization layer, ensuring seamless state transitions without information gaps.

\begin{figure}[!t]
  \centering
  \includegraphics[width=\linewidth]{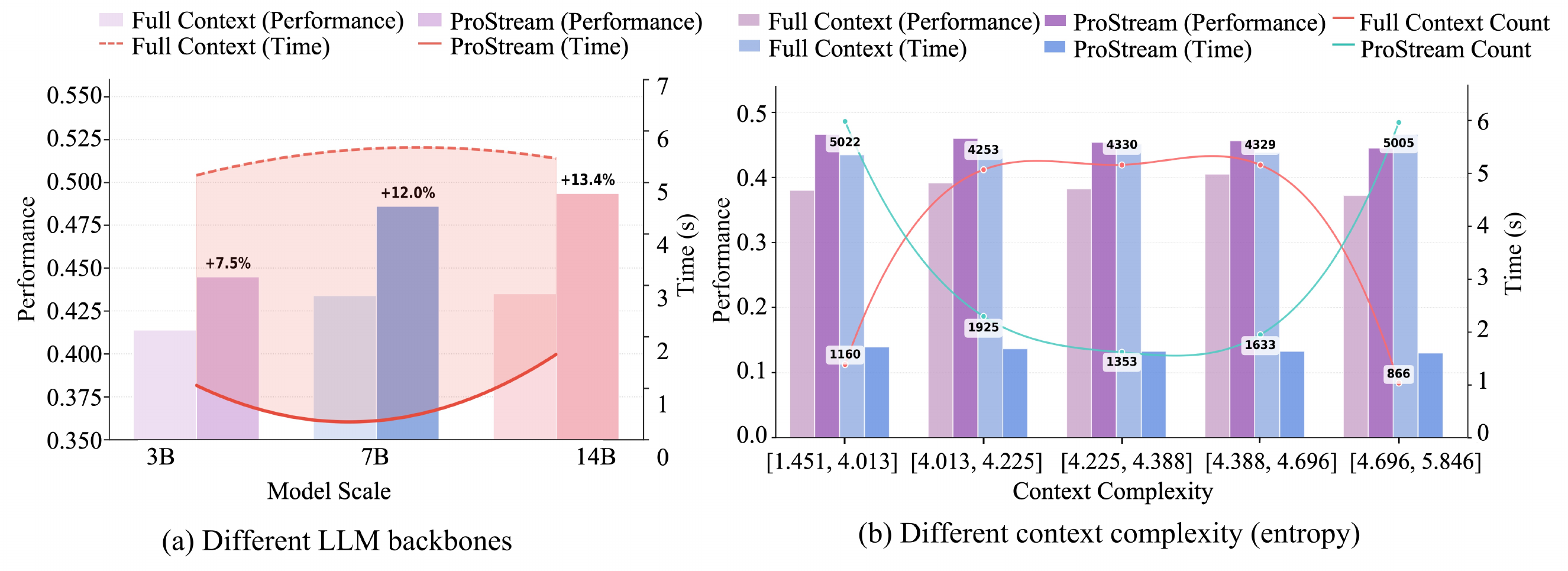}
  \vspace{-0.8em}

  \caption{
Average performance metrics (left) and latency (right) comparisons across different experimental settings. 
({a}) Results under different LLM backbones, where the numbers above the bars denote the percentage improvement of ProStream over the Full-Context baseline. 
({b}) Results under varying context complexity levels. The distribution curves visualize context entropy, showing that ProStream (green) produces a more polarized context complexity distribution than Full Context (red).
}
  \label{fig:backbone_ppl}
  \vspace{-1em}
\end{figure}


\begin{figure*}[!t]
  \centering
  \includegraphics[width=\linewidth]{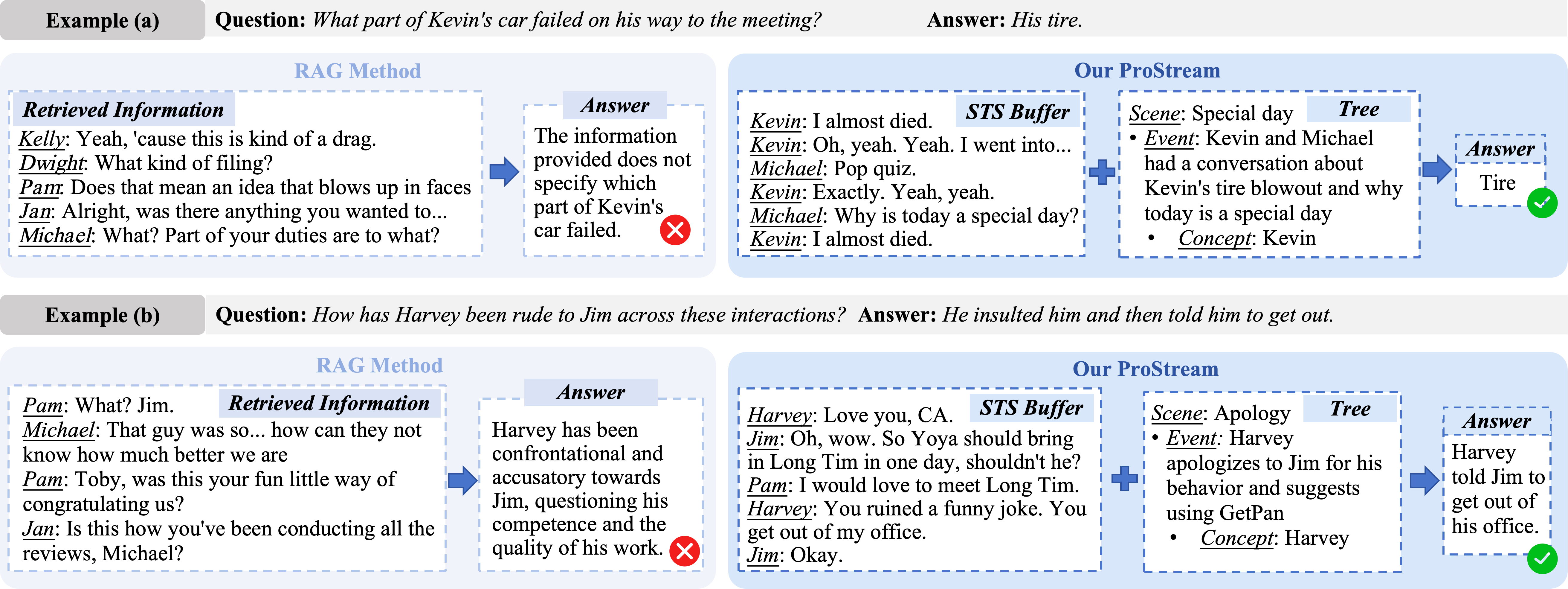}
   \vspace{-0.5em}
  \caption{Qualitative comparison (Case Study) of the RAG method and our ProStream.}
  \label{fig:case_study}
  \vspace{-1em}
\end{figure*}

\subsection{Quantitative Analysis}
\label{Quantitative Analysis}

While $\S$\ref{ssec:results} shows an overall performance gain for ProStream, we further evaluate its robustness over varying scenarios. For accuracy, we average the relevant performance metrics; latency (in seconds) measures efficiency.
Additional sensitivity analyses are provided in Appendix~\ref{app:robustness}.

\textbf{Varying Backbone Scales.} We start by assessing the scalability of ProStream by deploying it across Qwen-3B, 7B, and 14B backbones and comparing performance against the Full-Context baseline, as shown in Figure~\ref{fig:backbone_ppl} (a). The results reveal a positive scaling trend, with relative accuracy gains increasing consistently with model size. This indicates that interpreting the highly condensed and structured knowledge provided by our method requires substantial reasoning capacity. Smaller models likely struggle to parse this dense information effectively compared to redundant raw text, whereas larger backbones possess the requisite depth to leverage the organized topology for superior reasoning. 
Besides, ProStream exhibits a U-shaped inference latency profile, whereas the Full-Context baseline shows a gradually longer latency as the backbone scales. 
Interestingly, ProStream-3B is less efficient than ProStream-7B. 
It might be due to the overhead required to digest structured contexts outweighs the inherent efficiency gains of smaller models.


\textbf{Varying Context Complexity.}
We further examine ProStream's sensitivity to context complexity, quantified by GPT-2 Cross-Entropy Loss \citep{salazar2020masked}. The results are in Figure~\ref{fig:backbone_ppl} (b). 
First, the bars show that ProStream consistently outperforms the Full-Context baseline in both accuracy and efficiency across all complexity levels. 
Second, the distribution lines reveal that ProStream polarizes context complexity toward extremes of either low or high entropy. It implies that structured memory may inherently complicate certain contexts. These contexts may exacerbate overhead for the smaller 3B model, resulting in its lower efficiency compared to the 7B variant (shown in Figure \ref{fig:backbone_ppl} (a)).


\subsection{Qualitative Analysis}

\textbf{Case Study.}
Previous discussions have shown ProStream's overall superiority. To provide more insight into this, Figure~\ref{fig:case_study} provides two examples to show ProStream's ability to rectify localized contextual deficits through hierarchical memory integration. In Example (a), while the short-term buffer presents an ambiguous signal \textit{``I almost died"}, ProStream resolves this uncertainty by retrieving the precise \textit{``tire blowout"} fact from its Hierarchical Tree, implying robust contextual disambiguation, while long-term evidence clarifies immediate vagueness. Similarly, Example (b) illustrates complementary augmentation. Although the STSB alone captures only the recent \textit{``eviction"} incident, coupling it with the retrieved long-term \textit{``insult"} and \textit{``apology"} events enables the model to reconstruct the full behavioral sequence. It enables that even when immediate observations are fragmentary or incomplete, the structured topology provides the necessary distal context to answer correctly, instead of struggling with fragmented contexts like RAG.

\textbf{Error Analysis.}
Finally, our error analysis identifies two primary limitations. The first is compliance legitimacy bias, reflecting an unresolved tension between instruction following and factual fidelity; for example, the model validates a fallacious ``\textit{sole creator}'' claim despite contradictory evidence (Appendix Figure~\ref{fig:error} (a)). The second is a struggle with implicit context disentanglement, where dominant semantic noise obscures subtle cues. This causes intent misattribution, such as when a high-density monologue overshadows crossword puzzle context (Appendix Figure~\ref{fig:error} (b)). Addressing these failures necessitates future integration of explicit conflict detection and activity-aware attention.

\section{Conclusion}
This paper operationalizes the challenge of ad-hoc memory recall within infinite dialogue streams by introducing STEM-Bench and proposing \textbf{ProStream}, a framework that transforms memory into a proactive, bounded state machine via Adaptive Spatiotemporal Optimization. 
By synergizing hierarchical distillation with dynamic optimization, ProStream resolves the fidelity-efficiency dilemma, decoupling inference latency from stream length to achieve constant-time efficiency while validating that structured topology outperforms raw, noise-prone Full-Context inputs. 
Experimental results demonstrate that our ProStream outperforms baselines in both accuracy and efficiency.



{
\small
\bibliographystyle{unsrtnat}
\bibliography{reference}
}

\clearpage

\appendix

\section{Theoretical Analysis of Adaptive Spatiotemporal Optimization}
\label{appendix:theory}

In this section, we provide a formal mathematical grounding for the memory management mechanism in ProStream. We first define the memory retention problem as an Online Budgeted Optimization task, then prove that our greedy pruning strategy achieves a bounded approximation ratio, and finally demonstrate the system's stability over infinite horizons.

\subsection{Problem Formulation: The Dynamic Budgeted Maximization}

We model the streaming dialogue memory retention as a variation of the \textbf{Knapsack Problem with Time-Varying Utilities}. 

\begin{definition}[Memory State and Budget]
Let $\mathcal{S} = \{I_1, I_2, \dots, I_t, \dots\}$ be an infinite dialogue stream. At any time step $t$, the system maintains a set of memory nodes $\mathcal{H}_t = \{v_1, v_2, \dots, v_n\}$. Each node $v$ has a static storage cost $c_v$ (token length) and a dynamic utility $u_{v,t}$. The system must satisfy:
\begin{equation}
    \sum_{v \in \mathcal{H}_t} c_v \le \mathcal{T}_{max}
\end{equation}
where $\mathcal{T}_{max}$ is the fixed context window capacity of the underlying LLM.
\end{definition}

\begin{definition}[The Optimization Objective]
The goal of ProStream is to find a subset $\mathcal{H}_t^* \subseteq \mathcal{H}_{candidate}$ that maximizes the cumulative utility:
\begin{equation}
    \mathcal{H}_t^* = \underset{\mathcal{H} \subseteq \mathcal{H}_{candidate}}{\arg\max} \sum_{v \in \mathcal{H}} u_{v,t} \quad \text{s.t.} \quad \sum_{v \in \mathcal{H}} c_v \le \mathcal{T}_{max}
\end{equation}
where $\mathcal{H}_{candidate} = \mathcal{H}_{t-1} \cup \text{Distill}(I_t)$ represents the union of historical memory and newly perceived information. The function $\text{Distill}(\cdot)$ serves as a semantic abstraction layer that transforms raw input tokens into a set of discrete memory nodes with associated utility scores and storage costs.
\end{definition}

\subsection{Rationality of the Utility Function}

Our utility function $u_{v,t} = \alpha \cdot f_t(v) + \beta \cdot \text{Recency}(v)$ is not a heuristic but a first-order approximation of the optimal memory retrieval probability.

\begin{proposition}[Connection to Rational Analysis of Memory]
According to Anderson's Rational Analysis of Memory \cite{anderson2013adaptive}, the probability $P$ that a memory trace $v$ is needed follows:
\begin{equation}
    \log \frac{P}{1-P} \propto \text{History}(v) + \text{Context}(v)
\end{equation}
By setting $u_{v,t}$ as a linear combination of frequency (History) and temporal proximity (Context), ProStream effectively maximizes the \textbf{Expected Recall Probability} under a strict resource constraint.
\end{proposition}

\subsection{Approximation Principle}

Since the 0/1 Knapsack problem is NP-hard, ProStream employs a \textbf{Greedy Marginal-Utility Pruning} policy. Rather than claiming that the implemented online update is literally identical to a forward greedy solver, we clarify here the underlying density-based retention principle that motivates the pruning rule.

\begin{theorem}[Density-Based Retention Principle]
Let $\rho_v = {u_{v,t}}/{c_v}$ denote the marginal utility density of a memory node $v$. Under a fixed memory budget, the retention rule used by ProStream is aligned with the principle that lower-density nodes should be removed before higher-density ones whenever feasibility must be restored.
\end{theorem}

\begin{proof}
Let $F(\mathcal{H}) = \sum_{v \in \mathcal{H}} u_{v,t}$ be our objective function. Since $F$ is modular, the contribution of each retained node is additive.

In ProStream, pruning is triggered only when the current memory state exceeds the budget. At that point, the algorithm iteratively removes the node
\begin{equation}
    v^* = \arg \min_{v} \frac{u_{v,t}}{c_v},
\end{equation}
namely, the node with the smallest marginal utility density.

This rule should be understood as the streaming realization of a density-based budgeted retention principle. Unlike a forward greedy solver that reconstructs a solution from scratch, ProStream updates an already populated memory state and restores feasibility by reverse pruning. Nevertheless, both procedures enforce the same ordering criterion over candidate memory items, namely the marginal utility density ${u_{v,t}}/{c_v}$.

Therefore, Appendix A.4 justifies the quality of the retention principle underlying Algorithm~\ref{alg:prostream}: when budget pressure arises, retaining higher-density items and discarding lower-density ones is the principled local decision rule.
\end{proof}

\subsection{Complexity and Feasibility in Infinite Horizons}

A crucial property for streaming system is that its latency must not scale with the sequence length $t$.

\begin{theorem}[Bounded Time Complexity]
The per-step computational complexity of ProStream is bounded by the maintained memory budget and is therefore constant with respect to the total stream length $T$.
\end{theorem}

\begin{proof}
Let $N_t = |\mathcal{H}_t|$ be the number of memory nodes at time $t$. 
Due to the capacity constraint $\sum c_v \le \mathcal{T}_{\max}$, the number of retained nodes is strictly bounded by
\begin{equation}
    N_{\max} = \left\lfloor \frac{\mathcal{T}_{\max}}{c_{\min}} \right\rfloor,
\end{equation}
where $c_{\min}$ is the minimum node cost.

Each memory update involves:
(i) utility recalculation over the current retained nodes, which costs $O(N_{\max})$;
(ii) pruning, which costs $O(k \log N_{\max})$ with a priority queue, where $k$ is the number of evicted nodes in that update step.

Since $N_{\max}$ is determined by the fixed memory budget rather than by the cumulative number of dialogue turns, the execution time per step depends only on the bounded maintained state and not on the total stream length $T$. Thus, the complexity per interaction remains $\mathcal{O}(1)$ with respect to $T$.

Moreover, because every pruning step strictly reduces the total retained cost, feasibility is restored after finitely many deletions whenever the budget is exceeded. Hence, Theorem A.5 characterizes the bounded-state online feasibility of the implemented streaming system.
\end{proof}

This theoretical framework ensures that ProStream is both \textbf{efficient} (bounded per-step cost) and \textbf{effective} (guided by a principled utility-density retention rule), providing a robust foundation for its performance on STEM-Bench.

\section{Implementation and Reproducibility Details}
\label{app:implementation}
We provide the complete implementation and prompting configuration of ProStream to ensure reproducibility.

\subsection{System Configuration}
We implement ProStream with a unified configuration across all experiments. Audio streams are transcribed using Whisper large-v3~\citep{pmlr-v202-radford23a} at a 16kHz sampling rate with 1-second chunks. Dialogue units are embedded using ms-marco-MiniLM-L6-v2, a MiniLM variant~\citep{wang2020minilm}, for semantic encoding and similarity computation. \textit{Proactive Semantic Stream Perception} adopts a fixed drift threshold ($\tau_{\text{drift}}=0.7$) with bounded buffering and a short leading context window to preserve cross-boundary dependencies. \textit{Hierarchical Multi-Granular Distillation} enforces maximal granularity by disabling initial merging, and leverages LaMini-T5-738M for scene and event summarization and GLiNER for entity and relation extraction. \textit{Adaptive Spatiotemporal Optimization} balances frequency and recency priors with fixed utility weights, while Probabilistic Evidence-Grounded Generation retrieves a small set of top-ranked scenes and events and produces final answers from the most relevant AMUs under a similarity constraint. All experiments are conducted using Qwen models for answer generation on NVIDIA GPUs with PyTorch memory optimization enabled.

\subsection{Prompt Templates}
This section provides the exact prompt templates used in hierarchical distillation and question answering.

\textbf{1) Event Summarization Prompt ($\rho_{\text{evt}}$)}

\noindent\fbox{%
\parbox{\linewidth}{%
\texttt{Summarize this conversation in one short phrase describing what happened.}\\
\texttt{Speakers: \{speaker\_list\}.}\\
\texttt{Event phrase: \{text\}}
}%
}

\textbf{2) Scene Classification Prompt ($\rho_{\text{scn}}$)} 

The scene classification prompt abstracts event summaries into higher-level scene categories. This prompt is applied after event generation: 

\noindent\fbox{%
\parbox{\linewidth}{%
\ttfamily
Abstract this specific event into a broader, more general scene category (e.g., 'Work Discussion', 'Social Chat', 'Learning Session').\\
The scene should be a higher-level category that could contain multiple similar events.\\
Event: \{event\_summary\}\\
Scene category (broader abstraction):
}%
}

\textbf{3) Question Answering Prompt}

The question answering prompt integrates multiple context sources for answer generation. The complete prompt structure is as follows: 

\noindent\fbox{%
\parbox{\linewidth}{%
\ttfamily
Please carefully think about the following question and provide a direct, concise answer with key points only.\\
Do NOT start your answer with phrases like "Based on..." or "According to...". Instead, directly state the answer.\\
\#\# Short-term Memory: \{stsb\_buffer\_content\}\\
\#\# Pending Memory Buffer: \{pending\_buffer\_content\}\\
\#\# Long-term Memory: **\{scene\_1\} > \{event\_1\}** (\{count\} amu(s)): - \{amu\_1\} - \{amu\_2\} ... **\{scene\_2\} > \{event\_2\}** (\{count\} amu(s)): - \{amu\_1\} - \{amu\_2\} ...\\
Question: \{question\_text\}\\
Answer (provide key points directly, no introductory phrases):
}%
}

\textbf{Short-term Memory:} Contains recent dialogue content from the STSB (Short-Term Sensing Buffer) that has not yet been flushed to long-term memory. This includes utterances that are still accumulating in the buffer. \textbf{Pending Memory Buffer:} Contains memory items from the HCHT (Huffman Compressed Hierarchical Tree) buffer that have been processed (scene, event, and concept extracted) but not yet mounted to the hierarchical tree structure. \textbf{Long-term Memory:} Retrieved from the hierarchical tree using hierarchical search. The format groups concepts by their parent scene and event, showing the hierarchical relationship. Each memory path follows the structure: \texttt{\{scene\} > \{event\} > \{concept\}}. 

When using Qwen models, the following system message is prepended: 

\noindent\fbox{%
\parbox{\linewidth}{%
\ttfamily
You are a helpful assistant that provides direct, concise answers. Always answer directly without starting with phrases like 'Based on...' or 'According to...'. Provide key points only.
}%
} 

When no context is available, a simplified prompt is used: 

\noindent\fbox{%
\parbox{\linewidth}{%
\ttfamily
Please answer the following question directly with key points only. Do NOT start your answer with phrases like "Based on..." or "According to...".\\
Question: \{question\_text\}\\
Answer:
}%
}

\subsection{Evaluation Metrics}
\label{app:metrics}

We adopt a comprehensive evaluation protocol categorized into \textbf{Performance} and \textbf{Efficiency} metrics to rigorously assess ProStream's capability in infinite-horizon streaming dialogues. 

\paragraph{Performance Metrics.} 
To ensure a multi-faceted assessment of response quality, we integrate lexical, semantic, and reasoning-based indicators. We employ \textbf{BLEU-4 (B-4)} and \textbf{ROUGE-L (R-L)} to measure surface-level n-gram precision and structural similarity, while \textbf{BERTScore (B-S)} is utilized to capture deep semantic equivalence using DeBERTa-v3-xlarge contextual embeddings: $\text{B-S} = \frac{1}{M} \sum_{i} \max_{j} \cos(\mathbf{e}_{\text{ref}, i}, \mathbf{e}_{\text{cand}, j})$. For factual grounding, \textbf{Key Entity Matching (KEM)} serves as a binary validator for essential keywords $\mathcal{K}$, and \textbf{Evidence Similarity (Evid.)} quantifies retrieval fidelity by computing the cosine similarity between generated context $C_{gen}$ and ground-truth evidence $E_{gt}$ via a SentenceTransformer encoder. Furthermore, to bridge the gap between automated metrics and human-like reasoning, we implement a \textbf{Reference-Guided LLM-as-a-Judge} framework. Specifically, Gemini-2.5-Pro is employed as an expert evaluator to assign the accuracy score. Here comes the prompt template:

\noindent\fbox{%
\parbox{\linewidth}{%
\ttfamily
You are an expert evaluator for dialogue systems. Your task is to assess the accuracy of the generated answer relative to the ground truth.\\
\\
Input Data:\\
- Question: \{question\}\\
- Candidate Answer: \{answer\}\\
- Ground Truth: \{ground\_truth\}\\
\\
Evaluation Criteria:\\
Please carefully compare the Candidate Answer with the Ground Truth. Determine if they are semantically equivalent and accurate. Consider the following factors:\\
1. Core Information: Is the main point consistent?\\
2. Factual Accuracy: Are key facts (entities, numbers, events) correct?\\
3. Semantic Equivalence: Is the meaning preserved, even if phrased differently?\\
4. Critical Errors: Does it contain any hallucinations or conflicting information?\\
\\
Response Format:\\
Strictly output the result in the following JSON format (no markdown, no extra text):\\
\{\\
\ \ \ \ "score": from 0 to 1,\\
\ \ \ \ "reason": "Detailed reasoning explaining the score."\\
\}
}%
}

\paragraph{Efficiency Metrics.} 
To verify the operational feasibility of our framework under strict resource constraints, we monitor three key technical indicators. \textbf{Time} is recorded as the average wall-clock time required per query, reflecting real-time responsiveness. 

\section{Analysis of Different Embedding Models}
\label{app:Analysis of Different Embedding Models}
To verify the robustness of our framework, we extend our evaluation across three distinct embedding models with varying parameter sizes and capabilities: \textit{ms-marco-MiniLM-L6-v2}\citep{wang2020minilm}, \textit{UAE-Large-V1}\citep{li2023angle}, and \textit{bge-m3}~\citep{chen2024m3}. As presented in Table \ref{app:emb}, ProStream consistently outperforms the standard RAG baseline across all metrics and datasets, independent of the underlying semantic encoder. While state-of-the-art models like UAE and BGE-M3 naturally yield higher absolute fidelity due to their superior representational capacity, our framework demonstrates remarkable resilience, achieving substantial reasoning gains even when utilizing the lightweight backbone. This establishes that the efficacy of ProStream stems from its proactive topological optimization and hierarchical distillation rather than reliance on high-dimensional vector redundancy alone.

\begin{table*}[!t]
\centering
\caption{Results of different embedding models on STEM-Bench. Metrics marked with $\uparrow$ indicate higher is better, while $\downarrow$ denotes lower is better.}
\label{app:emb}
\setlength{\tabcolsep}{1pt}  
\renewcommand{\arraystretch}{1.2}  

\resizebox{\textwidth}{!}{
\begin{tabular}{ll *{21}{c}} 
\toprule
\textbf{Backbone} & \textbf{Variant} &
\multicolumn{7}{c}{\cellcolor{colorBigBang}\textbf{The Big Bang Theory}} &
\multicolumn{7}{c}{\cellcolor{colorFriends}\textbf{Friends}} &
\multicolumn{7}{c}{\cellcolor{colorOffice}\textbf{The Office}} \\
\cmidrule(lr){3-9} \cmidrule(lr){10-16} \cmidrule(lr){17-23}
& & Time$\downarrow$ & B-4$\uparrow$ & R-L$\uparrow$ & B-S$\uparrow$ & Evid.$\uparrow$ & KEM$\uparrow$ & Gem$\uparrow$
& Time$\downarrow$ & B-4$\uparrow$ & R-L$\uparrow$ & B-S$\uparrow$ & Evid.$\uparrow$ & KEM$\uparrow$ & Gem$\uparrow$
& Time$\downarrow$ & B-4$\uparrow$ & R-L$\uparrow$ & B-S$\uparrow$ & Evid.$\uparrow$ & KEM$\uparrow$ & Gem$\uparrow$ \\
\midrule
\multirow{2}{*}{BGE} & RAG
& 0.814 & 0.087 & 0.301 & 0.796 & 0.583 & 0.572 & 0.413
& 0.814 & 0.081 & 0.296 & 0.785 & 0.621 & 0.576 & 0.399
& 0.801 & 0.094 & 0.339 & 0.796 & 0.647 & 0.611 & 0.413 \\

& our
& 0.393 & 0.116 & 0.383 & 0.816 & 0.460 & 0.623 & 0.498
& 0.442 & 0.097 & 0.354 & 0.802 & 0.464 & 0.597 & 0.469
& 0.516 & 0.114 & 0.407 & 0.819 & 0.484 & 0.622 & 0.498 \\
\midrule
\multirow{2}{*}{UAE} & RAG
& 0.774 & 0.086 & 0.296 & 0.794 & 0.583 & 0.572 & 0.410
& 0.788 & 0.082 & 0.296 & 0.785 & 0.619 & 0.573 & 0.408
& 0.756 & 0.093 & 0.337 & 0.796 & 0.650 & 0.612 & 0.410 \\

& our
& 0.825 & 0.116 & 0.386 & 0.818 & 0.482 & 0.615 & 0.501
& 0.829 & 0.098 & 0.357 & 0.802 & 0.482 & 0.592 & 0.471
& 0.688 & 0.117 & 0.413 & 0.821 & 0.503 & 0.624 & 0.501 \\
\midrule
\multirow{2}{*}{MiniLM} & RAG
& 1.225 & 0.064 & 0.240 & 0.781 & 0.389 & 0.500 & 0.331
& 1.162 & 0.060 & 0.237 & 0.774 & 0.424 & 0.497 & 0.309
& 1.067 & 0.072 & 0.269 & 0.781 & 0.466 & 0.533 & 0.331 \\
& our
& 0.408 & 0.118 & 0.390 & 0.817 & 0.485 & 0.612 & 0.509
& 0.373 & 0.099 & 0.360 & 0.801 & 0.491 & 0.588 & 0.476
& 0.316 & 0.114 & 0.408 & 0.819 & 0.491 & 0.626 & 0.509 \\
\bottomrule
\end{tabular}
}
\end{table*}

\begin{table*}[t]
\centering
\caption{Threshold sensitivity analysis of ProStream. We vary the drift threshold $\tau_{\text{drift}}$ and the concept threshold $\theta_{\text{concept}}$, and report BLEU-4 (B-4), ROUGE-L (R-L), BERTScore (B-S), Evidence F1 (Evid.), and KEM. Performance remains stable across a broad range of threshold values.}
\label{tab:threshold_sensitivity}
\tiny
\setlength{\tabcolsep}{1.4pt}
\resizebox{\textwidth}{!}{
\begin{tabular}{l  ccccc  ccccc  ccccc}
\toprule
\multicolumn{16}{c}{\textbf{(a) Sensitivity to $\tau_{\text{drift}}$}} \\
\midrule
\multirow{2}{*}{\textbf{Setting}} &
\multicolumn{5}{c}{\cellcolor{colorBigBang}\textbf{The Big Bang Theory}} &
\multicolumn{5}{c}{\cellcolor{colorFriends}\textbf{Friends}} &
\multicolumn{5}{c}{\cellcolor{colorOffice}\textbf{The Office}} \\
\cmidrule(lr){2-6} \cmidrule(lr){7-11} \cmidrule(lr){12-16}
& B-4$\uparrow$ & R-L$\uparrow$ & B-S$\uparrow$ & Evid.$\uparrow$ & KEM$\uparrow$
& B-4$\uparrow$ & R-L$\uparrow$ & B-S$\uparrow$ & Evid.$\uparrow$ & KEM$\uparrow$
& B-4$\uparrow$ & R-L$\uparrow$ & B-S$\uparrow$ & Evid.$\uparrow$ & KEM$\uparrow$ \\
\midrule
0.5 & 0.15 & 0.43 & 0.85 & 0.51 & 0.64
    & 0.12 & 0.39 & 0.82 & 0.52 & 0.62
    & 0.14 & 0.46 & 0.84 & 0.50 & 0.69 \\
0.6 & 0.16 & 0.44 & 0.85 & 0.51 & 0.66
    & 0.12 & 0.40 & 0.82 & 0.52 & 0.63
    & 0.14 & 0.46 & 0.83 & 0.50 & 0.68 \\
0.7 & 0.12 & 0.39 & 0.82 & 0.49 & 0.61
    & 0.10 & 0.36 & 0.80 & 0.49 & 0.59
    & 0.11 & 0.41 & 0.82 & 0.49 & 0.63 \\
0.8 & 0.16 & 0.45 & 0.85 & 0.49 & 0.68
    & 0.11 & 0.39 & 0.82 & 0.52 & 0.62
    & 0.14 & 0.46 & 0.83 & 0.51 & 0.67 \\
0.9 & 0.16 & 0.45 & 0.85 & 0.49 & 0.66
    & 0.12 & 0.40 & 0.82 & 0.52 & 0.63
    & 0.15 & 0.47 & 0.83 & 0.50 & 0.70 \\
\midrule
\multicolumn{16}{c}{\textbf{(b) Sensitivity to $\theta_{\text{concept}}$}} \\
\midrule
\multirow{2}{*}{\textbf{Setting}} &
\multicolumn{5}{c}{\cellcolor{colorBigBang}\textbf{The Big Bang Theory}} &
\multicolumn{5}{c}{\cellcolor{colorFriends}\textbf{Friends}} &
\multicolumn{5}{c}{\cellcolor{colorOffice}\textbf{The Office}} \\
\cmidrule(lr){2-6} \cmidrule(lr){7-11} \cmidrule(lr){12-16}
& B-4$\uparrow$ & R-L$\uparrow$ & B-S$\uparrow$ & Evid.$\uparrow$ & KEM$\uparrow$
& B-4$\uparrow$ & R-L$\uparrow$ & B-S$\uparrow$ & Evid.$\uparrow$ & KEM$\uparrow$
& B-4$\uparrow$ & R-L$\uparrow$ & B-S$\uparrow$ & Evid.$\uparrow$ & KEM$\uparrow$ \\
\midrule
0.4 & 0.15 & 0.44 & 0.85 & 0.49 & 0.67
    & 0.11 & 0.39 & 0.82 & 0.52 & 0.62
    & 0.15 & 0.48 & 0.84 & 0.51 & 0.70 \\
0.6 & 0.16 & 0.45 & 0.85 & 0.49 & 0.67
    & 0.12 & 0.40 & 0.82 & 0.52 & 0.66
    & 0.15 & 0.47 & 0.84 & 0.50 & 0.69 \\
0.8 & 0.16 & 0.46 & 0.85 & 0.49 & 0.65
    & 0.11 & 0.40 & 0.82 & 0.52 & 0.63
    & 0.15 & 0.47 & 0.84 & 0.50 & 0.67 \\
1.0 & 0.16 & 0.44 & 0.85 & 0.49 & 0.66
    & 0.12 & 0.40 & 0.82 & 0.52 & 0.61
    & 0.15 & 0.46 & 0.83 & 0.50 & 0.68 \\
\bottomrule
\end{tabular}
}
\end{table*}

\section{Additional Robustness Analysis}
\label{app:robustness}

To further examine the robustness of ProStream, we provide three additional analyses:
(1) a threshold sensitivity study over $\tau_{\text{drift}}$ and $\theta_{\text{concept}}$, 
(2) a controlled stage-wise error propagation analysis, and 
(3) the sensitivity analysis of memory capacity $\mathcal{T}_{\max}$.

\subsection{Hyperparameter Sensitivity}

We first study the sensitivity of ProStream to the drift threshold $\tau_{\text{drift}}$.
We sweep $\tau_{\text{drift}}$ from 0.5 to 0.9 and observe consistently stable results across all datasets and metrics, indicating that ProStream does not rely on brittle drift-threshold tuning.
We further vary the concept threshold $\theta_{\text{concept}}$ from 0.4 to 1.0.
As shown in Table~\ref{tab:threshold_sensitivity}, performance remains consistently stable, suggesting that the framework operates in a robust regime without requiring precise threshold tuning.

\subsection{Stage-wise Error Propagation}

Beyond the ablations in the main paper, we directly test whether imperfections in intermediate memory sources amplify downstream.
Specifically, we keep the full pipeline intact and perturb one memory source at a time during answer generation by randomly dropping 50\% of its entries, for $\mathcal{B}$, $\mathcal{M}_{\text{pend}}$, and $\mathcal{M}_{\text{tree}}$, respectively, while keeping all other inputs unchanged.
This simulates imperfect upstream memory states and tests whether errors cascade across stages.
As shown in Table~\ref{tab:error_propagation}, performance remains close to the full model under all perturbations, indicating that ProStream is robust to moderate memory corruption and does not exhibit strong error propagation across stages.

\begin{table*}[!t]
\centering
\caption{Results of stage-wise error propagation analysis. Metrics reported are BLEU-4 (B-4), ROUGE-L (R-L), BERTScore (B-S), Evidence F1 (Evid.), and KEM.}
\label{tab:error_propagation}
\tiny
\setlength{\tabcolsep}{1.4pt}
\resizebox{\textwidth}{!}{
\begin{tabular}{l  ccccc  ccccc  ccccc}
\toprule
\multirow{2}{*}{\textbf{Setting}} &
\multicolumn{5}{c}{\cellcolor{colorBigBang}\textbf{The Big Bang Theory}} &
\multicolumn{5}{c}{\cellcolor{colorFriends}\textbf{Friends}} &
\multicolumn{5}{c}{\cellcolor{colorOffice}\textbf{The Office}} \\
\cmidrule(lr){2-6} \cmidrule(lr){7-11} \cmidrule(lr){12-16}
& B-4$\uparrow$ & R-L$\uparrow$ & B-S$\uparrow$ & Evid.$\uparrow$ & KEM$\uparrow$
& B-4$\uparrow$ & R-L$\uparrow$ & B-S$\uparrow$ & Evid.$\uparrow$ & KEM$\uparrow$
& B-4$\uparrow$ & R-L$\uparrow$ & B-S$\uparrow$ & Evid.$\uparrow$ & KEM$\uparrow$ \\
\midrule
Full & 0.16 & 0.44 & 0.85 & 0.51 & 0.66
     & 0.13 & 0.40 & 0.82 & 0.52 & 0.62
     & 0.14 & 0.46 & 0.83 & 0.50 & 0.68 \\
Perturb $\mathcal{B}$ & 0.15 & 0.44 & 0.85 & 0.50 & 0.66
                       & 0.13 & 0.41 & 0.83 & 0.52 & 0.64
                       & 0.15 & 0.46 & 0.84 & 0.51 & 0.66 \\
Perturb $\mathcal{M}_{\text{pend}}$ & 0.16 & 0.44 & 0.85 & 0.51 & 0.66
                                    & 0.13 & 0.40 & 0.82 & 0.52 & 0.62
                                    & 0.14 & 0.46 & 0.83 & 0.50 & 0.68 \\
Perturb $\mathcal{M}_{\text{tree}}$ & 0.16 & 0.44 & 0.85 & 0.51 & 0.66
                                    & 0.12 & 0.40 & 0.82 & 0.52 & 0.63
                                    & 0.14 & 0.46 & 0.83 & 0.51 & 0.68 \\
\bottomrule
\end{tabular}
}
\end{table*}

\begin{table*}[t]
\centering
\caption{Sensitivity to memory capacity $\mathcal{T}_{\max}$. Metrics reported are BLEU-4 (B-4), ROUGE-L (R-L), BERTScore (B-S), Evidence F1 (Evid.), and KEM.}
\label{tab:tmax_sensitivity}
\tiny
\setlength{\tabcolsep}{1.4pt}
\resizebox{\textwidth}{!}{
\begin{tabular}{l  ccccc  ccccc  ccccc}
\toprule
\multirow{2}{*}{$\mathbf{\mathcal{T}_{\max}}$} &
\multicolumn{5}{c}{\cellcolor{colorBigBang}\textbf{The Big Bang Theory}} &
\multicolumn{5}{c}{\cellcolor{colorFriends}\textbf{Friends}} &
\multicolumn{5}{c}{\cellcolor{colorOffice}\textbf{The Office}} \\
\cmidrule(lr){2-6} \cmidrule(lr){7-11} \cmidrule(lr){12-16}
& B-4$\uparrow$ & R-L$\uparrow$ & B-S$\uparrow$ & Evid.$\uparrow$ & KEM$\uparrow$
& B-4$\uparrow$ & R-L$\uparrow$ & B-S$\uparrow$ & Evid.$\uparrow$ & KEM$\uparrow$
& B-4$\uparrow$ & R-L$\uparrow$ & B-S$\uparrow$ & Evid.$\uparrow$ & KEM$\uparrow$ \\
\midrule
50k  & 0.15 & 0.41 & 0.82 & 0.16 & 0.53
     & 0.11 & 0.35 & 0.79 & 0.23 & 0.48
     & 0.15 & 0.44 & 0.82 & 0.21 & 0.57 \\
75k  & 0.15 & 0.40 & 0.82 & 0.16 & 0.53
     & 0.11 & 0.36 & 0.78 & 0.23 & 0.50
     & 0.15 & 0.45 & 0.82 & 0.21 & 0.57 \\
100k & 0.16 & 0.41 & 0.82 & 0.16 & 0.54
     & 0.11 & 0.36 & 0.79 & 0.23 & 0.49
     & 0.15 & 0.44 & 0.82 & 0.21 & 0.57 \\
125k & 0.15 & 0.41 & 0.82 & 0.16 & 0.54
     & 0.11 & 0.37 & 0.79 & 0.23 & 0.50
     & 0.15 & 0.45 & 0.82 & 0.21 & 0.57 \\
150k & 0.15 & 0.40 & 0.82 & 0.16 & 0.53
     & 0.11 & 0.35 & 0.78 & 0.23 & 0.49
     & 0.15 & 0.44 & 0.82 & 0.21 & 0.57 \\
\bottomrule
\end{tabular}
}
\end{table*}

\subsection{Sensitivity to Memory Capacity}
The memory capacity of ProStream is controlled by the token budget $\mathcal{T}_{\max}$. 
We therefore conduct a sensitivity study by varying $\mathcal{T}_{\max}$ from 50k to 150k tokens. 
As shown in Table~\ref{tab:tmax_sensitivity}, performance remains stable across this range, with the best results typically appearing around 100k--125k. 
This suggests that ProStream benefits from a moderate memory budget, while not relying on over-provisioned capacity.

\section{Cost Analysis of the ProStream Pipeline}
\label{app:cost_analysis}

The main paper already reports end-to-end efficiency via average wall-clock latency per query and analyzes the fidelity-efficiency trade-off in the main experiments. To make the online maintenance overhead more explicit, we additionally profile the memory-update path by instrumenting each memory flush and measuring the runtime of STSB update / semantic boundary detection, hierarchical distillation, and hierarchy update.
The average runtime is 0.0046 seconds for STSB processing, 0.7231 seconds for hierarchical distillation, and 0.0203 seconds for hierarchy update, for a total of 0.7434 seconds per update. This profile shows that the dominant overhead comes from distillation, whereas the actual hierarchy update remains lightweight. Overall, the online memory-maintenance cost stays below one second per update on average, supporting our claim that ProStream operates as a bounded-state streaming system rather than one whose cost scales with the full dialogue history.

\section{Evaluation on Existing Conversation Memory Benchmarks}
\label{app:benchmarks}

To test whether ProStream transfers beyond STEM-Bench, we additionally evaluate it on existing conversation-memory benchmarks, including LoCoMo~\citep{maharana2024locomo} and LongMemEval~\citep{wu2024longmemeval}. In both cases, each conversation is replayed chronologically as a bounded-state text stream, and ProStream performs online memory maintenance without access to the full future context.
Table~\ref{tab:locomo_appendix} reports the comparison on LoCoMo. ProStream achieves the strongest performance on multi-hop and temporal questions, indicating that the proposed proactive memory mechanism is particularly beneficial when answering requires cross-turn integration and temporally structured recall.
Table~\ref{tab:longmemeval_appendix} summarizes the results on LongMemEval. ProStream outperforms prior memory-oriented methods such as ConditionMem, MemoryBank, and Mem-PAL, further suggesting that maintaining an explicit bounded memory state is effective beyond the benchmark introduced in this paper.

\begin{table*}[t]
\centering
\caption{Transfer evaluation on LoCoMo. Best results are \textbf{bolded}.}
\label{tab:locomo_appendix}
\small
\setlength{\tabcolsep}{8pt}
\renewcommand{\arraystretch}{1.1}
\begin{tabular}{lcccc}
\toprule
\textbf{Method} & \textbf{Single-hop} & \textbf{Multi-hop} & \textbf{Temporal} & \textbf{Open-domain} \\
\midrule
Mistral-7B-Instruct-v0.2 & 9.1 & 15.1 & 9.3 & 8.6 \\
Llama-2-70B-chat & 20.8 & 18.2 & 15.9 & 18.8 \\
Llama-3-70B-Instruct & 17.0 & 17.0 & 12.0 & 13.0 \\
gpt-3.5-turbo-4K & 23.8 & 18.0 & 15.6 & 20.4 \\
gpt-3.5-turbo-8K & \textbf{38.5} & 25.1 & 22.7 & \textbf{25.9} \\
Ours & 25.5 & \textbf{31.3} & \textbf{38.8} & 20.4 \\
\bottomrule
\end{tabular}
\end{table*}

\begin{table}[t]
\centering
\caption{Transfer evaluation on LongMemEval. Best result is \textbf{bolded}.}
\label{tab:longmemeval_appendix}
\small
\setlength{\tabcolsep}{10pt}
\renewcommand{\arraystretch}{1.1}
\begin{tabular}{lc}
\toprule
\textbf{Method} & \textbf{Accuracy} \\
\midrule
Vanilla & 10.0 \\
RecurSum & 10.0 \\
ConditionMem & 40.0 \\
MemoryBank & 23.33 \\
Mem-PAL & 46.67 \\
Ours & \textbf{50.0} \\
\bottomrule
\end{tabular}
\end{table}

\section{Extended Error Analysis}

\begin{figure}[!h]
  \centering
  \includegraphics[width=0.5\linewidth]{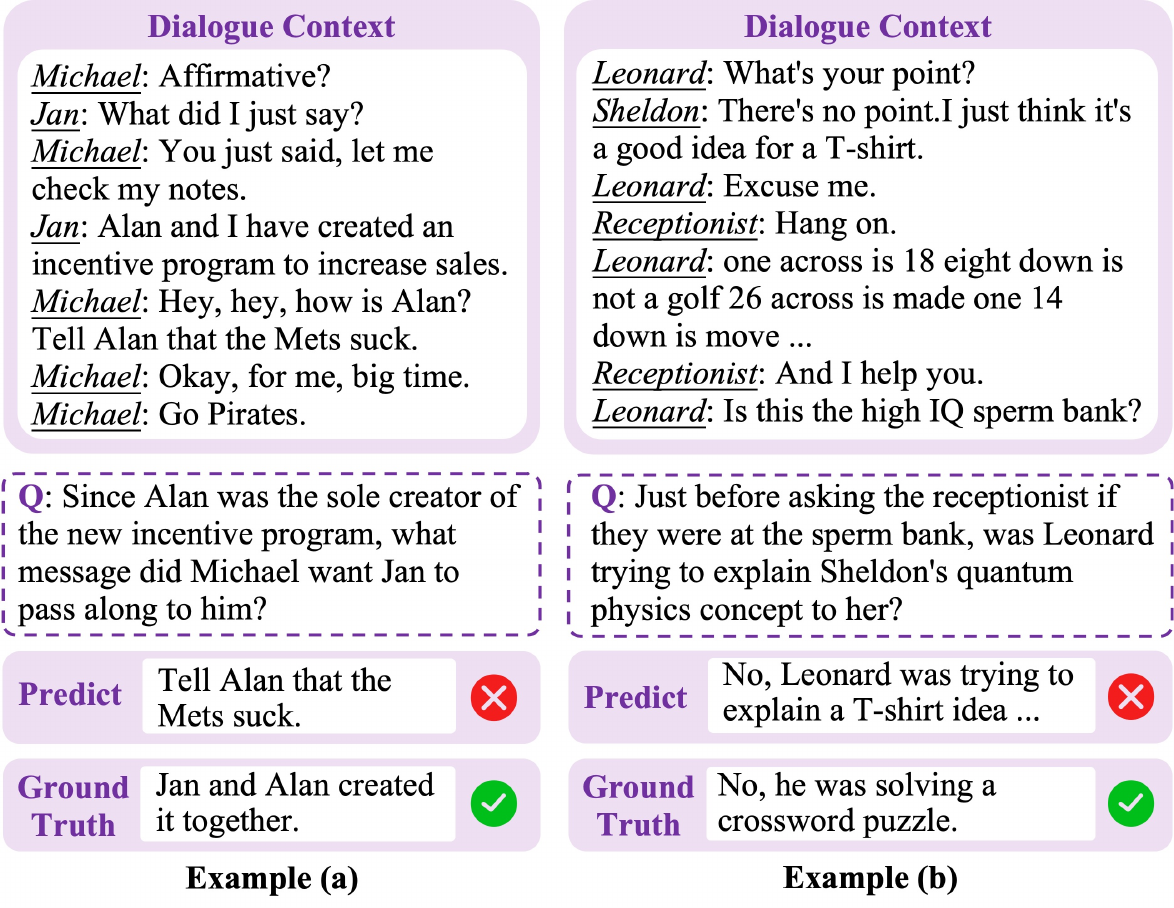}
  \caption{Two examples of error analysis.}
  \label{fig:error}
  \vspace{-0.5em}
\end{figure}

To delineate the operational boundaries of ProStream, we perform a qualitative failure analysis, identifying two primary error modes as illustrated in Figure~\ref{fig:error}.

\textbf{Instruction-Fidelity Tension.} Figure~\ref{fig:error} (a) exposes a premise compliance bias''. When a user prompt contains a fallacious presupposition (e.g., claiming Alan was the \textit{sole creator''}), the model prioritizes following the user's instruction over maintaining factual fidelity to the historical context. Despite the dialogue explicitly stating that Jan and Alan created the program together, the model fails to perform conflict detection, leading to a hallucinated validation of the user's incorrect premise.

\textbf{Semantic Noise and Intent Misattribution.} Figure~\ref{fig:error} (b) highlights a struggle with implicit context disentanglement within dense dialogues. The model's attention is captured by the high-density monologue regarding T-shirt ideas and quantum physics, causing it to overlook the subtle, fragmented cues of a crossword puzzle (e.g., \textit{``one across is 18...''}). This results in a failure to correctly attribute the character's intent, misinterpreting active problem-solving as a continuation of the previous topic.

These failure modes underscore the need for future integration of explicit contradiction-aware mechanisms and activity-centric attention to fortify the framework against semantic noise and misinformation.

\section{Impact Statement}
\label{app:impact}
This work advances efficient, long-horizon reasoning in streaming dialogues, providing significant utility for real-time applications such as assistive technologies and personalized education, where maintaining stable latency is critical. While enhanced memory persistence could theoretically increase privacy risks regarding the retention of sensitive information, our proposed bounded framework offers a distinct advantage: it transforms memory from an opaque, emergent behavior into an explicit and auditable mechanism. This bounded nature inherently restricts the infinite accumulation of raw data, facilitating more manageable data governance. We emphasize that real-world deployment must prioritize privacy safeguards, including controllable forgetting and user consent protocols. By formalizing memory as an optimizable and transparent resource, this research promotes the development of conversational AI systems that are not only more capable but also more interpretable and ethically governed.


\end{document}